\documentclass[letterpaper,10pt,conference]{ieeeconf}
\pdfoutput=1
\usepackage{cite}
\usepackage{graphicx}
\usepackage{subcaption}
\usepackage[cmex10]{amsmath}
\usepackage{xcolor}
\usepackage{array}
\usepackage{url}
\usepackage{bbm}
\usepackage{todonotes}

\usepackage{algorithm}
\usepackage{algorithmic}
\usepackage{bm}

\usepackage{amsmath,amssymb}

\usepackage{amsmath} 
\usepackage{bm}      
\usepackage{epsfig}
\usepackage{amsfonts,amssymb,multirow,bigstrut,booktabs,ctable,latexsym}
\usepackage{mathtools}
\usepackage[mathscr]{eucal}
\usepackage{verbatim}

\usepackage{amsthm}
\usepackage{algorithm}
\usepackage[normalem]{ulem}

\usepackage{amsmath}
\newtheorem{thm}{Theorem}[section]

\newtheorem{lm}[thm]{Lemma}

\newtheorem{df}[thm]{Definition}

\newtheorem{prop}[thm]{Proposition}

\newtheorem{ass}[thm]{Assumption}

\usepackage{amssymb}
\usepackage[switch]{lineno}

\usepackage{amsthm}

\IEEEoverridecommandlockouts

\begin{document}
	
	\title{\bf Shaping Advice in Deep Multi-Agent Reinforcement Learning}
	
	\author{Baicen Xiao$^{1}$, Bhaskar Ramasubramanian$^{1}$, and Radha Poovendran$^{1}$%
		\thanks{$^{1}$Network Security Lab, Department of Electrical and Computer Engineering, 
			University of Washington, Seattle, WA 98195, USA. \newline
			{\tt\small \{bcxiao, bhaskarr, rp3\}@uw.edu}}
	}
	
\maketitle
	
\begin{abstract}
Multi-agent reinforcement learning involves multiple agents interacting with each other and a shared environment to complete tasks. 
When rewards provided by the environment are sparse, agents may not receive immediate feedback on the quality of actions that they take, thereby affecting learning of policies. 
In this paper, we propose a method called \emph{Shaping Advice in deep Multi-agent reinforcement learning (SAM)} to augment the reward signal from the environment with an additional reward termed \emph{shaping advice}. 
The shaping advice is given by a difference of potential functions at consecutive time-steps. Each potential function is a function of observations and actions of the agents. 
The shaping advice needs to be specified only once at the start of training, and can be easily provided by non-experts. 
We show through theoretical analyses and experimental validation that the shaping advice provided by SAM does not distract agents from completing tasks specified by the environment reward. 
Theoretically, we prove that the convergence of policy gradients and value functions when using SAM implies the convergence of these quantities in the same environment in the absence of SAM. 
Experimentally, we evaluate SAM on three tasks in the \emph{multi-agent Particle World} environment that have sparse rewards. We observe that using SAM results in agents learning policies to complete tasks faster, and obtain higher rewards than: i) using sparse rewards alone; ii) a state-of-the-art reward redistribution method. 
\end{abstract}
	
\section{Introduction}
 
Multi-agent reinforcement learning (MARL) involves multiple autonomous agents, all of whom share a common environment \cite{busoniu2008comprehensive}. 
Each agent learns to complete tasks by maximizing a cumulative reward, where the reward signal is provided by the environment. 
Examples of multi-agent systems where MARL has been applied include autonomous vehicle coordination \cite{sallab2017deep}, multi-player video games \cite{tampuu2017multiagent}, and analysis of social dilemmas \cite{leibo2017multi}. 

In these settings, any single agent interacts not only with the environment, but also with other agents. 
As the behaviors of agents evolve in the environment, the environment will become non-stationary from the perspective of any single agent. 
Thus, agents that independently learn behaviors by assuming other agents to be part of the environment can result in unstable learning regimes \cite{foerster2017stabilising, matignon2012independent, tan1993multi}. 

When trained agents are deployed independently, or when communication among agents is costly, the agents will need to be able to learn decentralized policies. 
Decentralized policies can be efficicently learned by adopting the \emph{centralized training with decentralized execution (CTDE)} paradigm, first introduced in \cite{lowe2017multi}. 
An agent using CTDE can make use of information about other agents' observations and actions to aid its own learning during training, but will have to take decisions independently at test-time. 
However, the ability of an agent to learn decentralized policies can be affected if reward signals from the environment are sparse. 

The availability of immediate feedback on the quality of actions taken by the agents at each time-step is critical to the learning of behaviors to successfully complete a task. 
This is termed \emph{credit assignment} \cite{sutton2018reinforcement}. 
Sparse rewards make it difficult to perform effective credit assignment at intermediate time-steps of the learning process. 
One approach that has been shown to improve the learning of policies when rewards are sparse is \emph{reward shaping} \cite{agogino2008analyzing, devlin2011empirical, devlin2014potential}. 
Reward shaping techniques augment the reward provided by the environment with an additional \emph{shaping reward}. 
The shaping reward can be designed to be \emph{dense} (i.e., not sparse), and agents learn policies using the augmented reward.

There needs to be a systematic approach to provide the shaping rewards, 
since the additional reward can distract an agent from completing the task specified by the reward provided by the environment \cite{randlov1998learning}. 
In this paper, we refer to the additional reward given to agents at each time-step as \emph{\textbf{shaping advice}}. 
The shaping advice is specified by a difference of potential functions at consecutive time-steps, where each potential function depends on observations and actions of the agents. 
Potential-based methods 
ensure that agents will not be distracted from completing a task specified by the reward given by the environment since 
the total potential-based reward obtained when starting from a state and returning to the same state at a future time is zero \cite{ng1999policy}.

In this paper, we introduce \emph{Shaping Advice in deep Multi-agent reinforcement learning (SAM)}, a method that 
incorporates information about the task and environment to define \emph{shaping advice}. 
The advice in SAM can be interpreted as \emph{domain knowledge} that aids credit assignment \cite{mannion2018reward}. 
This advice only needs to be specified once at the start of the training process. 
We demonstrate that SAM does not distract agents from completing tasks specified by the reward from the environment. 
We make the following contributions in this paper: 
\begin{itemize}
\item We introduce SAM to incorporate potential-based shaping advice in multi-agent environments with continuous states and actions. SAM uses the CTDE paradigm to enable agents to efficiently learn decentralized policies. 
\item We theoretically demonstrate that SAM does not distract agents by proving that convergence of policy gradients and values when using SAM implies convergence of these quantities in the absence of SAM. 
\item We verify our theoretical results by evaluating SAM on three tasks in the multi-agent Particle World environment \cite{lowe2017multi}. These tasks include cooperative and competitive objectives, and have sparse rewards. We show that using SAM allows agents to learn policies to complete the tasks faster, and obtain higher rewards than: i) using sparse rewards alone, and ii) a state-of-the-art reward redistribution technique. 
\end{itemize}
This paper extends techniques introduced in our previous work \cite{xiao2019potential} for single-agent RL to the multi-agent setting. 

The remainder of this paper is organized as follows: Section \ref{RelatedWork} presents related work and Section \ref{Background} provides an introduction to stochastic games and policy gradients. We provide details on SAM and present our theoretical analysis on its convergence in Section \ref{Methods}. Experiments validating the use of SAM are reported in Section \ref{Experiments}, and Section \ref{Conclusion} concludes the paper.

\section{Related Work}\label{RelatedWork}

Decentralized and distributed control techniques for multi-agent systems is a popular area of research. 
A widely studied problem in such systems is the development of algorithms to specify methods by which information can be exchanged among agents so that they can jointly complete tasks.
A technique to ensure fixed-time consensus for multi-agent systems whose interactions were specified by a directed graph was studied in \cite{zuo2015nonsingular, tian2018fixed}. 
The authors of \cite{li2020consensus} proposed an adaptive distributed event-triggering protocol to guarantee consensus for multi-agent systems specified by linear dynamics and interactions specified by an undirected graph. 
We direct the reader to \cite{qin2016recent} for a survey of recent developments in consensus of multi-agent systems. 
These works, however, assumed the availability of models of individual agent's dynamics, and of the interactions between agents. 

In the absence of a model of the environment, techniques have been proposed to train agents to complete tasks in cooperative MARL tasks where all agents share the same global reward.  
The authors of \cite{sunehag2018value} introduced value decomposition networks that decomposed a centralized value into a  
sum of individual agent values to assess contributions of individual agents to a shared global reward. 
An additional assumption on monotonicity of the centralized value function 
was imposed in QMIX  \cite{rashid2018qmix} to assign credit to an agent. 
The action spaces of agents in the above-mentioned works were discrete and finite, and these techniques cannot be easily adapted to settings with continuous action spaces. 
In comparison, we study reward shaping in cooperative and competitive MARL tasks in environments with continuous action spaces.
 
In single agent RL, feedback signals provided by a human operator have been used to improve credit assignment. 
Demonstrations provided by a human operator were used to synthesize a ‘baseline policy’ that was used to guide learning in \cite{taylor2011integrating, wang2017improving}. 
When expert demonstrations were available, imitation learning was used to guide exploration of the RL agent in \cite{kelly2019hg, ross2011reduction}. 
Feedback provided by a human operator was converted to a shaping reward to aid training a deep RL agent in environments with delayed rewards in \cite{xiao2020fresh}. 

An alternative approach is potential-based reward shaping. 
Although this requires prior knowledge of the problem domain, potential-based techniques have been shown to offer 
guarantees on optimality and convergence of policies in both single \cite{ng1999policy} and multi-agent \cite{devlin2011empirical, devlin2011theoretical, lu2011policy} cases. 
These works had focused on the use of potential-based methods in environments with discrete action spaces. 
In our previous work \cite{xiao2019potential}, we had developed potential-based techniques to learn stochastic policies in single-agent RL with continuous states and actions. 
We adapt and extend methods from \cite{xiao2019potential} to develop potential-based techniques for MARL in this paper. 

In \cite{gangwani2020learning}, the authors presented a method called iterative relative credit refinement (IRCR). 
This method used a `surrogate objective' to uniformly redistribute a sparse reward along the length of a trajectory in single and multi-agent RL. 
We empirically compare SAM with IRCR, and explain why SAM is able to guide agents to learn policies that result in higher average rewards than IRCR. 

\section{Background}\label{Background}

This section establishes the notation that will be used in the paper, and provides a brief introduction to stochastic games and multi-agent policy gradients. 

\subsection{Stochastic Games}

A \emph{stochastic game} with $n$ players is a tuple $\mathcal{G} = (X, A^1, \dots, A^n, \mathbb{T}, R^1, \dots, R^n, O^1,\dots,O^n,\rho_0, \gamma)$. $X$ is the set of states, $A^i$ is the action set of player $i$, 
$\mathbb{T}: X \times A^1 \times \dots \times A^n \times X \rightarrow [0,1]$ encodes $\mathbb{P}(x_{t+1}|x_t, a^1_t, \dots, a^n_t)$, the probability of transition to state $x_{t+1}$ from $x_t$, given the respective player actions. 
$R^i: X \times A^1 \times \dots \times A^n \rightarrow \mathbb{R}$ is the reward obtained by agent $i$ when transiting from $x_t$ while each player takes action $a^i_t$. 
$O^i$ is the set of observations for agent $i$. 
At every state, each agent receives an observation correlated with the state: $o^i : X \rightarrow O^i$. 
$\rho_0$ is a distribution over the initial states, and $\gamma \in [0,1]$ is a discounting factor. 

A \emph{policy} for agent $i$ is a distribution over actions, defined by $\pi^i: O^i \times A^i \rightarrow [0,1]$. 
Let $\bm{\pi}:=\{\pi^1, \dots, \pi^n\}$ and $s:=(s^1,\dots,s^n)$. 
Following \cite{lowe2017multi}, in the simplest case, $s^i = o^i$ for each agent $i$, and we use this for the remainder of the paper. 
Additional information about states of agents can be included since we compute \emph{centralized} value functions. 
Let $V_i ^{\bm{\pi}} (s)=V_i (s, \pi^1, \dots, \pi^n):= \mathbb{E}_{\bm{\pi}} [\sum_t \gamma^t R^i_t | s_0 = s, \bm{\pi}]$ 
and $Q_i ^{\bm{\pi}} (s, a^1,\dots,a^n):=R^i + \gamma \mathbb{E}_{s'}[V_i ^{\bm{\pi}} (s')]$ where $V_i ^{\bm{\pi}} (s)= \mathbb{E}_{\{a^i \sim \pi^i\}_{i=1}^n}[Q_i ^{\bm{\pi}} (s, a^1,\dots,a^n)]$.  

\subsection{Multi-agent Policy Gradient}

Policy gradient methods compute (an estimate of) the gradient of the expected reward with respect to the policy parameters. 
This is used to improve the policy by \emph{ascending} in the direction of the gradient.  

Assume that the policy $\pi^i$ for agent $i$ is parameterized by $\theta_i$. 
We denote this by $\pi_{\theta_i}$, and define $\bm{\pi_\theta}:=\{\pi_{\theta_1},\dots,\pi_{\theta_n}\}$. 
We assume that $\pi_{\theta_i}(a^i|o_i) > 0$ for all $\theta_i$ and is continuously differentiable with respect to $\theta_i$.  
The value of the parameterized policy $\bm{\pi_\theta}$ for agent $i$ is then: $J_i(\bm{\theta}):= \mathbb{E}_{\bm{\pi_\theta}} [\sum_t \gamma^t R^i_t]$. 

We use $-i$ to denote all agents other than agent $i$. 
Define the \emph{accumulated return} for agent $i$ from time $t$ onwards as $G_i(s_t,a^i_t,a^{-i}_t) : = \sum_{j=t}^\infty \gamma ^{j-t} R^i_j$.  Then, from the policy gradient theorem \cite{sutton2018reinforcement}:
\begin{align}
\nabla_{\theta_i} J_i(\bm{\theta}) &= \mathbb{E}_{\bm{\pi_\theta}} [ \nabla_{\theta_i} \log~\pi_{\theta_i} (a^i_t|o^i_t)~G_i(s_t,a^i_t,a^{-i}_t)]. \nonumber
\end{align}
When $G_i(s_t,a^i_t,a^{-i}_t)$ is replaced by $Q^{\bm{\pi_\theta}}_i(s_t,a^i_t,a^{-i}_t)= \mathbb{E}_{\bm{\pi_\theta}}[G_i(s_t,a^i_t,a^{-i}_t)]$ in the above equation, this is called an \emph{actor-critic} \cite{konda2000actor, sutton2018reinforcement}. 
In actor-critic methods, the \emph{critic} estimates a value function ($Q^{\bm{\pi_\theta}}_i(s_t,a^i_t,a^{-i}_t)$) and 
the \emph{actor} (policy) is learned by following a gradient that depends on the critic. 
$Q^{\bm{\pi_\theta}}_i(s_t,a^i_t,a^{-i}_t)$ can be further replaced by $Q^{\bm{\pi_\theta}}_i(s_t,a^i_t,a^{-i}_t) - b_i(s_t)$, where $b_i(s_t)$ is a baseline used to reduce the variance. 
When $b_i(s_t) = V ^{\bm{\pi_\theta}}_i(s_t)$, this difference is called the \emph{advantage}. 
The authors of \cite{lowe2017multi} extended actor-critic methods to work with deterministic policies, and termed their approach  multi-agent deep deterministic policy gradient (MADDPG). 

\section{Shaping Advice in Multi-Agent RL}\label{Methods}
\begin{figure}[!h]
	\centering
	\includegraphics[width=\linewidth]{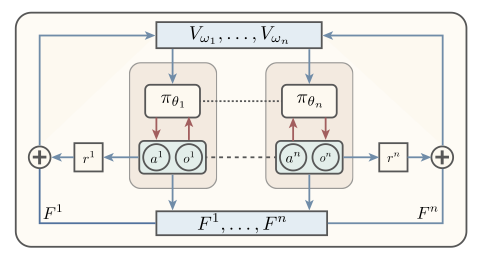} 
	\caption{Schematic of SAM. A centralized critic estimates value functions $V_{\omega_1},\dots,V_{\omega_n}$. Actions for an agent $i$ are sampled from its policy $\pi_{\theta_i}$ in a decentralized manner. Actions and observations of all agents are used to determine \emph{shaping advice} $F^1,\dots,F^n$. The advice $F^i$ is augmented to the reward $r^i$ from the RL environment. The workflow shown by blue arrows in the outer box is required only during training. During execution, only the workflow shown by the red arrows inside the inner boxes is needed.}\label{SAMARITANSchematic}
\end{figure}

This section introduces \emph{\textbf{SAM}} to augment \emph{shaping advice} to the reward supplied by the MARL environment to provide immediate feedback to agents on their actions. 
SAM uses the CTDE paradigm wherein agents share parameters with each other during the training phase, but execute decentralized policies using only their own observations at test-time. 
Figure \ref{SAMARITANSchematic} shows a schematic of SAM. 
We subsequently detail how the shaping advice is provided to the agents, and analyze the optimality and convergence of policies when using SAM.

\subsection{Shaping Advice}

Reward shaping augments the reward supplied by the environment ($R$) of the RL agent with an additional reward ($F$) called the \emph{shaping reward}, which typically encodes information about the RL environment. 
The shaping reward is \emph{potential-based} if it is given by the difference of potentials \cite{ng1999policy, wiewiora2003principled}. 
Potential-based methods enable incorporation of shaping rewards in a principled manner. 
This ensures that an agent will not be distracted from the task specified by the environment reward since the total potential-based reward obtained when starting from a state and returning to the same state at a future time is zero (when $\gamma = 1$) \cite{ng1999policy}. 

In the multi-agent case, \emph{shaping advice} for an agent at each time is a function of observations and actions of all agents. 
The shaping advice is augmented to the environment reward during training, 
and can take one of two forms, 
\emph{look-ahead} and \emph{look-back}, respectively given by: 
\begin{align}
&F^i_t(s_t,a^i_t,a^{-i}_t,s_{t+1},a^i_{t+1},a^{-i}_{t+1})\label{LAPBA}\\ &\qquad \qquad :=\gamma \phi_i(s_{t+1},a^i_{t+1},a^{-i}_{t+1}) - \phi_i(s_t,a^i_t,a^{-i}_t)\nonumber 
\end{align}
\begin{align}
&F^i_t(s_t,a^i_t,a^{-i}_t,s_{t-1},a^i_{t-1}, a^{-i}_{t-1})\label{LBPBA}\\&\qquad \qquad:=\phi_i(s_t,a^i_t,a^{-i}_t) - \gamma^{-1} \phi_i(s_{t-1},a^i_{t-1}, a^{-i}_{t-1})\nonumber
\end{align} 
We will denote by $\mathcal{G}'$ the $n$ player stochastic game that is identical to $\mathcal{G}$, but with rewards $R'^{i}:=R^i + F^i$ for each $i$. 

The shaping advice is a heuristic that uses knowledge of the environment and	 task, along with information available to the agent \cite{gupta2017cooperative}. 
For example, in the particle world tasks that we study, each agent has access to positions of other agents and of landmarks, relative to itself. 
This is used to design shaping advice for individual agents at each time step. 

\subsection{Centralized Critic}

SAM uses a centralized critic during the training phase. 
Information about observations and actions of all agents is used to learn a decentralized policy for each agent. 
One way to do this is by using an actor-critic framework, which combines policy gradients with \emph{temporal difference (TD)} techniques. 

At time $t$, the joint action $(a^1_t,\dots,a^n_t)$ is used to estimate the accumulated return for each agent $i$ as $r^i_t + \gamma V^i (s_{t+1})$. 
This quantity is called the \emph{TD-target}. 
Subtracting $V^i(s_t)$ from the TD-target gives the \emph{TD-error}, which is an unbiased estimate of the agent advantage \cite{sutton2018reinforcement}. 
Each actor can then be updated following a gradient based on this TD-error. 

We learn a separate critic for each agent like in \cite{lowe2017multi}. However, the learning process can be affected when rewards provided by the environment are sparse. 
SAM uses a \emph{potential-based} heuristic as shaping advice that is augmented to the reward received from the environment.
This resulting reward signal is less sparse and can be used by the agents to learn policies.

\subsection{Shaping Advice in Multi-Agent Actor-Critic}

We describe how to augment shaping advice to the multi-agent policy gradient to assign credit. 
We use the actor-critic framework with a centralized critic and decentralized actors. 

For an agent $i$, shaping advice $F^i$ is augmented to the environment reward $r^i$ at each time step. 
$F^i$ is specified by a difference of potentials (Eqn. (\ref{LAPBA}) or (\ref{LBPBA})). 
The centralized critic allows using observations and actions of all agents to specify $F^i$. 
Using look-ahead advice, $Q$-values in the modified game $\mathcal{G}'$ with reward $R^i+F^i$ and original game $\mathcal{G}$ with reward $R^i$ are related as \cite{wiewiora2003principled}:
\begin{align}
[Q^{\bm{\pi_\theta}}_i(s_t,a^i_t,a^{-i}_t)]_{\mathcal{G}}&=[Q^{\bm{\pi_\theta}}_i(s_t,a^i_t,a^{-i}_t)]_{\mathcal{G}'}\nonumber\\&\qquad \qquad+\phi_i(s_t,a^i_t,a^{-i}_t) \label{RelnGG'}
\end{align}

The accumulated return in $\mathcal{G}'$ for agent $i$ is then estimated by $r^i_t+ \gamma \phi_i(s_{t+1},a^i_{t+1},a^{-i}_{t+1}) - \phi_i(s_t,a^i_t,a^{-i}_t)+ \gamma V^i (s_{t+1})$. 
From Equation (\ref{RelnGG'}), we can add $\phi_i(s_t,a^i_t,a^{-i}_t)$ to the TD-target in $\mathcal{G}'$ at each time step to keep the policy gradient unbiased in $\mathcal{G}$. 

Let the critic and actor in SAM for agent $i$ be respectively parameterized by $\omega_i$ and $\theta_i$.
When the actor is updated at a slower rate than critic, the asymptotic behavior of the critic can be analyzed by keeping the actor fixed using \emph{two time-scale stochastic approximation} methods \cite{borkar2009stochastic}. 
For agent $i$, the TD-error at time $t$ is given by:
\begin{align}
\delta^i_t&:=r^i_t+F^i_t+\gamma V_{\omega_i} (s_{t+1}) - V_{\omega_i} (s_t). \label{TDErrori}
\end{align}
The update of the critic can be expressed as a first-order ordinary differential equation (ODE) in $\omega_i$, given by: 
\begin{align}
\dot{\omega}_i &= \mathbb{E}_{\bm{\pi_\theta}}[\delta^i \nabla_{\omega_i} V_{\omega_i} (s_t)] \label{CriticODE}
\end{align}
Under an appropriate parameterization of the value function, this ODE will converge to an asymptotically stable equilibrium, denoted $\omega_i (\bm{\theta})$. 
At this equilibrium, the TD-error for agent $i$ is $\delta_{t,\omega_i(\bm{\theta})}^i = r^i_t+F^i_t +\gamma V_{\omega_i(\bm{\theta})} (s_{t+1}) - V_{\omega_i(\bm{\theta})} (s_t)$. 

The update of the actor can then be determined by solving a first order ODE in $\theta_i$. 
With look-ahead advice, a term corresponding to the shaping advice at time $t$ will have to be added to ensure an unbiased policy gradient (Eqn. (\ref{RelnGG'})). 
This ODE can be written as: 
\begin{align}
\dot{\theta}_i &= \mathbb{E}_{\bm{\pi_\theta}}[(\delta_{t,\omega_i(\bm{\theta})}^i +\phi_i(s_t,a^i_t,a^{-i}_t))  \nabla_{\theta_i} \log~\pi_{\theta_i}(a_t^i|o^i_t)] \label{ActorODE}
\end{align} 
A potential term will not have to be added to ensure an unbiased policy gradient when utilizing look-back advice. This insight follows from Proposition 3 in \cite{xiao2019potential} since we consider decentralized policies. 

\subsection{Analysis}

In this part, we present a proof of the convergence of the actor and critic parameters when learning with shaping advice. 
We also demonstrate that convergence of policy gradients and values when using SAM implies convergence of these quantities in the absence of SAM. 
This will guarantee that policies learned in the modified stochastic game $\mathcal{G}'$ will be locally optimal in the original game $\mathcal{G}$. 

For agent $i$, the update dynamics of the critic can be expressed by the ODE in Eqn. (\ref{CriticODE}). 
Assuming parameterization of $V(s)$ over a linear family, this ODE will converge to an asymptotically stable equilibrium \cite{borkar2009stochastic}. 
The actor update is then given by the ODE in Eqn. (\ref{ActorODE}). 
The parameters associated with the critics are assumed to be updated on a faster timescale than those of the actors. 
Then, the behaviors of the actor and critic can be analyzed separately using two timescale stochastic approximation techniques \cite{borkar2009stochastic}. 

\begin{ass}
	We make the following assumptions:
	\begin{enumerate}
		\item At any time $t$, an agent is aware of the actions taken by all other agents. Rewards received by the agents at each time step are uniformly bounded.
		\item The Markov chain induced by the agent policies is irreducible and aperiodic. 
		\item For each agent $i$, the update of its policy parameter $\theta_i$ includes a projection operator $\Gamma_i$, which projects $\theta_i$ onto a compact set $\Theta_i$. We assume that $\Theta_i$ includes a stationary point of $\nabla_{\theta_i} J_i(\bm{\theta})$ for each $i$. 
		\item For each agent $i$, its value function is parameterized by a linear family. That is, $V_{\omega_i} (s) = \Phi_i \omega_i$, where $\Phi_i$ is a known, full-rank feature matrix for each $i$. 
		\item For each agent $i$, the TD-error at each time $t$ and the gradients $\nabla_{\omega_i} V_{\omega_i}(s)$ are bounded, and the gradients 
		$\nabla_{\theta_i} \log ~\pi_{\theta_i} (\cdot|s_t)$ are Lipschitz with bounded norm. 
		\item The learning rates satisfy $\sum_t \alpha^\theta_t = \sum_t \alpha^\omega_t = \infty$, $\sum_t[(\alpha^\theta_t)^2 + (\alpha^\omega_t)^2] < \infty$, $\lim_{t \rightarrow \infty} \frac{\alpha^\theta_t}{\alpha^\omega_t}=0$.
	\end{enumerate}
\end{ass}
We first state a useful result from \cite{kushner2012stochastic}.
\begin{lm}[\cite{kushner2012stochastic}]\label{KushnerClarkLemma}
	Let $\Gamma:\mathbb{R}^k \rightarrow \mathbb{R}^k$ be a projection onto a compact set $K \subset \mathbb{R}^k$. 
	Define 
	\begin{align}
	\hat{\Gamma}(h(x)):&=\lim_{\epsilon \downarrow 0} \frac{\Gamma(x+\epsilon h(x))-x}{\epsilon} \nonumber
	\end{align}
	for $x \in K$ and $h:\mathbb{R}^k \rightarrow \mathbb{R}^k$ continuous on $K$. Consider the update $x_{t+1} = \Gamma (x_t + \alpha_t(h(x_t) + \xi_{t,1}+\xi_{t,2}))$ and its associated ODE $\dot{x} = \hat{\Gamma}(h(x))$. 
	Assume that: 
	
	i) $\{\alpha_t\}$ is such that $\sum_t \alpha_t = \infty$, $\sum_t \alpha_t^2 < \infty$; 
	
	ii) $\{\xi_{t,1}\}$ is such that for all $\epsilon > 0$, $\lim_t \mathbb{P} (\sup_{n \geq t} || \sum_{\tau = t}^n \alpha_{\tau} \xi_{\tau,1}|| \geq \epsilon) = 0$; 
	
	iii) $\{\xi_{t,2}\}$ is an almost surely bounded random sequence, and $\xi_{t,2} \rightarrow 0$ almost surely. 

	Then, if the set of asymptotically stable equilibria of the ODE in $\dot{x}$ is compact, denoted $K_{eq}$, the updates $x_{t+1}$ will converge almost surely to $K_{eq}$. 
\end{lm}
Let $\{\mathcal{F}^\omega_t\}$ be the filtration where $\mathcal{F}^\omega_t:=\sigma(s_{\tau},$ $r^1_{\tau},\dots,r^n_{\tau},\omega_{1_\tau},\dots,\omega_{n_{\tau}}:\tau \leq t)$ is an increasing $\sigma-$algebra generated by iterates of $\omega_i$ up to time $t$. 
We first analyze behavior of the critic when parameters of the actor are fixed. 

\begin{thm}\label{ThmCriticConv}
	For a fixed policy $\bm{\pi_\theta}$, the update $\omega_i \leftarrow \omega_i - \alpha_t^\omega \delta^i_t \nabla_{\omega_i} V_{\omega_i} (s_t)$ converges almost surely to the set of asymptotically stable equilibria of the ODE $\dot{\omega}_i=h_i(\omega_i):=\mathbb{E}_{\bm{\pi_\theta}}[\delta^i_t \nabla_{\omega_i}V_{\omega_i}(s_t)|\mathcal{F}^\omega_t]$. 
\end{thm}

\begin{proof}
	Let $\xi^i_{t,1}:=\delta^i_t \nabla_{\omega_i} V_{\omega_i} (s_t)-\mathbb{E}_{\bm{\pi_\theta}}[\delta^i_t \nabla_{\omega_i}V_{\omega_i}(s_t)|\mathcal{F}^\omega_t]$. 
	Then, the $\omega_i$ update can be written as $\omega_i \leftarrow \omega_i- \alpha_t^\omega[h_i(\omega_i)+\xi^i_{t,1}]$, where  
	$h_i(\omega_i)$ is continuous in $\omega_i$. 
	Since $\delta^i_t$ and $\nabla_{\omega_i} V_{\omega_i}(s)$ are bounded, $\xi^i_{t,1}$ is almost surely bounded. 
	
	Let $M^i_t:= \sum_{\tau = 0}^t \alpha^\omega_{\tau}  \xi^i_{\tau,1}$. 
	Then $\{M^i_t\}$ is a martingale\footnote{A martingale \cite{williams1991probability} is a stochastic process $S_1,S_2,\dots$ that satisfies $\mathbb{E}(|S_n| < \infty)$ and $\mathbb{E}(S_{n+1}|S_1,\dots,S_n) = S_n$ for each $n = 1,2,\dots.$.}
	, and $\sum_t ||M^i_t-M^i_{t-1}||^2 = \sum_t ||\alpha^\omega_t \xi^i_{t,1}||^2 < \infty$ almost surely. 
	Therefore, from the martingale convergence theorem \cite{williams1991probability}, the sequence $\{M^i_t\}$ converges almost surely. 
	Therefore, the conditions in Lemma \ref{KushnerClarkLemma} are satisfied. 
	
	Since $V_{\omega_i} = \Phi_i \omega_i$, with $\Phi_i$ a full-rank matrix, $h_i(\omega_i)$ is a linear function, and the ODE will have a unique equibrium point. 
	This will be an asymptotically stable equilibrium since ODE dynamics will be governed by a matrix of the form $(\gamma T_{\pi} - I)$. 
	Here, $I$ is an identity matrix, and $T_{\pi}$ is a stochastic state-transition matrix under policy $\pi$, whose eigen-values have (strictly) negative real parts \cite{prasad2015two}. 
	Denote this asymptotically stable equilibrium by $\omega_i(\bm{\theta})$. 
\end{proof}
We can now analyze the behavior of the actor, assuming that the critic parameters have converged to an asymptotically stable equilibrium. 
With $\omega_i (\bm{\theta})$ a limit point of the critic update, let $\delta_{t,\omega_i(\bm{\theta})}^i = r^i_t+F^i_t +\gamma V_{\omega_i(\bm{\theta})} (s_{t+1}) - V_{\omega_i(\bm{\theta})} (s_t)$. 
When using look-ahead or look-back advice, define $\tilde{\delta}^i_{t,\omega_i(\bm{\theta})}$ as:
\begin{align}
\begin{split}
&\text{look-ahead:}\quad\tilde{\delta}^i_{t,\omega_i(\bm{\theta})}:= (\delta_{t,\omega_i(\bm{\theta})}^i+\phi_i(s_t,a^i_t,a^{-i}_t)) \\
&\text{look-back:} \quad\tilde{\delta}^i_{t,\omega_i(\bm{\theta})}:= \delta_{t,\omega_i(\bm{\theta})}^i
. 
\end{split}\label{delta}
\end{align}

Let $\{\mathcal{F}^\theta_t\}$ be a filtration where $\mathcal{F}^\theta_t:=\sigma(\bm{\theta}_{\tau}:=[\theta_{1_\tau}\dots \theta_{n_\tau}]:\tau \leq t)$ is an increasing $\sigma-$algebra generated by iterates of $\theta_i$ up to time $t$. 

\begin{thm} \label{ThmActorConv} 
	The update $\theta_i \leftarrow \Gamma_i[\theta_i + \alpha_t^\theta \tilde{\delta}^i_t $ $ \nabla_{\theta_i} \log~\pi_{\theta_i}(a_t^i|o^i_t)]$ converges almost surely to the set of asymptotically stable equilbria of the ODE $\dot{\theta}_i = \hat{\Gamma}_i(h_i(\theta_i))$, where $h_i(\theta_i) = \mathbb{E}_{\bm{\pi_\theta}}[\tilde{\delta}^i_{t,\omega_i(\bm{\theta})}\nabla_{\theta_i} \log~\pi_{\theta_i}(a_t^i|o^i_t)|\mathcal{F}^\theta_t]$.
\end{thm}

\begin{proof}
	Let $\xi^i_{t,1} := \tilde{\delta}^i_t \nabla_{\theta_i} \log \pi_{\theta_i}(a_t^i|o^i_t) - \mathbb{E}_{\bm{\pi_\theta}}[\tilde{\delta}^i_t \nabla_{\theta_i} \log\pi_{\theta_i}(a_t^i|o^i_t) | \mathcal{F}^\theta_t]$ and $\xi^i_{t,2}:= \mathbb{E}_{\bm{\pi_\theta}}[(\tilde{\delta}^i_t - \tilde{\delta}^i_{t,\omega_i(\bm{\theta})})\nabla_{\theta_i} \log \pi_{\theta_i}(a_t^i|o^i_t) | \mathcal{F}^\theta_t]$. 
	Then, the update of $\theta_i$ can be written as $\theta_i \leftarrow \theta_i + \alpha^\theta_t [h_i(\theta_i) +\xi^i_{t,1} + \xi^i_{t,2}]$, where 
	$h_i(\theta_i)$ is continuous in $\theta_i$. We now need to verify that the conditions in Lemma \ref{KushnerClarkLemma} are satisfied. 
	
	Since the critic parameters converge almost surely to a fixed point, $\tilde{\delta}^i_t - \tilde{\delta}^i_{t,\omega_i(\bm{\theta})} \rightarrow 0$ almost surely. Therefore, $\xi^i_{t,2} \rightarrow 0$ almost surely, verifying 
	Condition iii) in Lemma \ref{KushnerClarkLemma}. 
	
	Since $\tilde{\delta}^i_t$ and $\nabla_{\theta_i} \log \pi_{\theta_i}(a_t^i|o^i_t)$ are bounded, $\xi^i_{t,1}$ is continuous in $\theta_i$ and $\theta_i$ belongs to a compact set, the sequence $\{\xi^i_{t,1}\}$ is bounded almost surely \cite{rudin1964principles}. 
	If $M^i_t:= \sum_{\tau = 0}^t \alpha^\theta_{\tau}  \xi^i_{\tau,1}$, then $\{M^i_t\}$ is a martingale, and $\sum_t ||M^i_t-M^i_{t-1}||^2 = \sum_t ||\alpha^\theta_t \xi^i_{t,1}||^2 < \infty$ almost surely. 
	Then, 
	$\{M^i_t\}$ converges almost surely \cite{williams1991probability}, satisfying  
	Condition ii) of Lemma \ref{KushnerClarkLemma}. 
	Condition i) is true by assumption, completing the proof.
\end{proof} 
Theorems \ref{ThmCriticConv} and \ref{ThmActorConv} demonstrate the convergence of critic and actor parameters in the stochastic game with the shaped reward, $\mathcal{G}'$. 
However, our objective is to provide a guarantee of convergence in the original game $\mathcal{G}$. 
We establish such a guarantee when parameterizations of the value function results in small errors, and policy gradients in $\mathcal{G}'$ are bounded. 

\begin{df}
	For a probability measure $\mu$ on a finite set $\mathcal{M}$, the $\ell_2-$norm of a function $f$ with respect to $\mu$ is defined as $||f||_{\mu}:=\bigg[\int_{\mathcal{M}} |f(X)|^2 d\mu (X)\bigg]^{\frac{1}{2}}=\bigg[\mathbb{E}_{\mu}(|f(X)|^2)\bigg]^{\frac{1}{2}}$. 
\end{df}
\begin{prop}\label{BoundInOrigGame}
	In the stochastic game $\mathcal{G}'$, let  $(\gamma+1)||V_i^{\pi_{\bm{\theta}}}(s) - V_{\omega_i(\bm{\theta})}(s)||_{\pi_{\bm{\theta}}} \leq \mathcal{E}_i(\bm{\theta})$, and let $||\nabla_{\theta_i} \log\pi_{\theta_i}||_{\pi_{\bm{\theta}}} \leq C_i(\bm{\theta})$. 
	Let $(\bm{\theta}^*, \omega(\bm{\theta})^*)$ be the set of limit points of SAM. 
	
	Then, in the original stochastic game $\mathcal{G}$, for each agent $i$, $||\nabla_{\theta_i}J_i(\bm{\theta}^*)||_2 \leq C_i(\bm{\theta}^*)\mathcal{E}_i(\bm{\theta}^*)$.
\end{prop}
\begin{proof}
	Let $\Theta_{i_{eq}}$ denote the set of asymptotically stable equilibria of the ODE in $\theta_i$. 
	Let $\Theta_{eq}:=\Theta_{1_{eq}} \times \dots \times \Theta_{n_{eq}}$. 
	Then, in the set $\Theta_{eq}$, $\dot{\theta}_i = 0$ for each agent $i$. 
	
	Consider a policy $\pi_{\bm{\theta}}$, $\bm{\theta} \in \Theta_{eq}$. 
	In the original game $\mathcal{G}$, 
	\begin{align}
	\nabla_{\theta_i}J_i(\bm{\theta}) &= \mathbb{E}_{\bm{\pi_\theta}} [ \nabla_{\theta_i} \log\pi_{\theta_i} (a^i_t|o^i_t)Q^{\bm{\pi_\theta}}_i(s_t,a^i_t,a^{-i}_t)]\label{MAPolGrad}
	\end{align}
	
	From Equation (\ref{RelnGG'}), $[Q^{\bm{\pi_\theta}}_i(s_t,a^i_t,a^{-i}_t)]_{\mathcal{G}}=[Q^{\bm{\pi_\theta}}_i(s_t,a^i_t,a^{-i}_t)]_{\mathcal{G}'}$ $+\phi_i(s_t,a^i_t,a^{-i}_t)$. 
	Since we use an advantage actor critic, we replace $[Q^{\bm{\pi_\theta}}_i(s_t,a^i_t,a^{-i}_t)]_{\mathcal{G}'}$ with an advantage term, defined as $[Q^{\bm{\pi_\theta}}_i(s_t,a^i_t,a^{-i}_t)]_{\mathcal{G}'}-V^{\bm{\pi_\theta}}_i(s_t)$. 
	Substituting these quantities in Equation (\ref{MAPolGrad}), 
	\begin{align}
	\nabla_{\theta_i}J_i(\bm{\theta}) &= \mathbb{E}_{\bm{\pi_\theta}}[\nabla_{\theta_i} \log\pi_{\theta_i} (a^i_t|o^i_t).\label{MAPolGradG'}\\&\qquad \qquad (r^i_t+F^i_t+\gamma V_i^{\pi_{\bm{\theta}}}(s_{t+1})\nonumber\\&\qquad \qquad-V^{\bm{\pi_\theta}}_i(s_t)+ \phi_i(s_t,a^i_t,a^{-i}_t))]\nonumber
	\end{align}
	
	At equilibrium,  $\dot{\theta}_i= 0$ in Equation (\ref{ActorODE}). 
	Subtracting this from Equation (\ref{MAPolGradG'}), 
	\begin{align}
	&\nabla_{\theta_i}J_i(\bm{\theta}) - \dot{\theta}_i =\nabla_{\theta_i}J_i(\bm{\theta}) \nonumber\\
	&=\mathbb{E}_{\bm{\pi_\theta}}[\nabla_{\theta_i} \log\pi_{\theta_i} (a^i_t|o^i_t). \nonumber \\&\qquad \qquad(\gamma (V_i^{\pi_{\bm{\theta}}}(s_{t+1}) - V_{\omega_i(\bm{\theta})} (s_{t+1}))\nonumber\\&\qquad \qquad- (V_i^{\pi_{\bm{\theta}}}(s_{t}) - V_{\omega_i(\bm{\theta})}(s_t)))] \nonumber
	\end{align} 
	
	Using the Cauchy-Schwarz inequality, 
	\begin{align}
	||\nabla_{\theta_i}J_i(\bm{\theta}^*)||_2 &\leq |\gamma +1| .||V_i^{\pi_{\bm{\theta}}}(s) - V_{\omega_i(\bm{\theta})}(s)||_{\pi_{\bm{\theta}}}.\nonumber\\&\qquad \qquad \qquad ||\nabla_{\theta_i} \log\pi_{\theta_i}||_{\pi_{\bm{\theta}}}\nonumber\\
	&\leq C_i(\bm{\theta}^*)\mathcal{E}_i(\bm{\theta}^*)\label{CauchSchwIneq}
	\end{align}
	
	Each term on the right side of Eqn. (\ref{CauchSchwIneq}) is bounded. Thus, $J_i(\bm{\theta})$ converges for each agent $i$ in the original game $\mathcal{G}$, even though policies are synthesized in the modified game $\mathcal{G}'$. 
\end{proof}

Proposition \ref{BoundInOrigGame} demonstrates that the additional reward $F^i$ provided by SAM to guide the agents does not distract them from accomplishing the task objective that is originally specified by the environment reward $R^i$. 

\subsection{Algorithm}
Algorithm \ref{algo:MAAC+PBA} desecribes SAM. 
The shaping advice is specified as a difference of potential functions (\emph{Line 15}), and is added to the reward received from the environment. 
We use an advantage-based actor-critic, and use the TD-error to estimate this advantage (\emph{Line 16}). This is used to update the actor and critic parameters for each agent (\emph{Lines 18-19}).
\begin{algorithm}[!h]
	\small
	\caption{SAM: Shaping Advice in deep Multi-agent RL}
	\label{algo:MAAC+PBA}
	\begin{algorithmic}[1]

		\REQUIRE{For each agent $i$: parameters $\theta_i$ (for agent policy), $\omega_i$ (for agent value function); Shaping advice $\phi_i(s,a^i,a^{-i})$.\\\quad \hspace{1mm}	 Learning rates $\alpha^\theta, \alpha^\omega$; Episode limit $T_{max}$.}
		\STATE{$T = 0$}
		\REPEAT
		\STATE{$t \leftarrow -1$; $\phi_i(s_{-1},a^i_{-1},a^{-i}_{-1}) = 0$ for all $i$}
		\STATE{Initialize information $s_0 = [o_0^1,\dots,o_0^n]$}
		\REPEAT
		\STATE{$t \leftarrow t+1$}
		\FOR{agent $i=1$ to $n$}
		\STATE{sample $a^i_t \sim \pi_{\theta_i}(\cdot|o^i_t)$}
		\ENDFOR
		\STATE{Take action $a_t = [a^1_t,\dots,a^n_t]$, observe new information $s_{t+1}$ and obtain reward $r^i_t$ for each agent. Use $a_t$ to determine $\phi_i(s_t,a_t)$ for all agents}
		\IF{$s_{t+1}$ is terminal}
		\STATE{$V_{\omega_i} (s_{t+1}) =0$}
		\ENDIF
		\FOR{agent $i=1$ to $n$}
		\STATE{compute $F^i_t$ based on equations (\ref{LAPBA}) and (\ref{LBPBA})}
		\STATE{TD-error: $\delta^i_t:=r^i_t+F^i_t+\gamma V_{\omega_i} (s_{t+1}) - V_{\omega_i} (s_t)$}
		\STATE{compute $\tilde{\delta}^i_t$ based on equations (\ref{delta})}
		\STATE{Update actor: $\theta_i \leftarrow \Gamma_i[\theta_i + \alpha_t^\theta \tilde{\delta}^i_t  \nabla_{\theta_i} \log~\pi_{\theta_i}(a_t^i|o^i_t)]$}
		\STATE{Update critic: $\omega_i \leftarrow \omega_i - \alpha_t^\omega \delta^i_t \nabla_{\omega_i} V_{\omega_i} (s_t)$}
		\ENDFOR
		\UNTIL{$s_{t+1}$ is terminal}
		\STATE{$T \leftarrow T+1$}
		\UNTIL{$T>T_{max}$}
	\end{algorithmic}
\end{algorithm}

\section{Experiments}\label{Experiments}
This section describes the multi-agent tasks that we evaluate SAM on, and these include tasks with cooperative and competitive objectives. 
In each case, the rewards provided to the agents are sparse, which affects the agents' ability to obtain immediate feedback on the quality of their actions at each time-step. 
Shaping advice provided by SAM is used to guide the agents to obtain higher rewards than in the case without advice. 
We conclude the section by presenting the results of our experiments evaluating SAM on these tasks. 
\subsection{Task Descriptions and Shaping Advice}
\begin{figure}[!h]
	\centering
	\includegraphics[width=\linewidth]{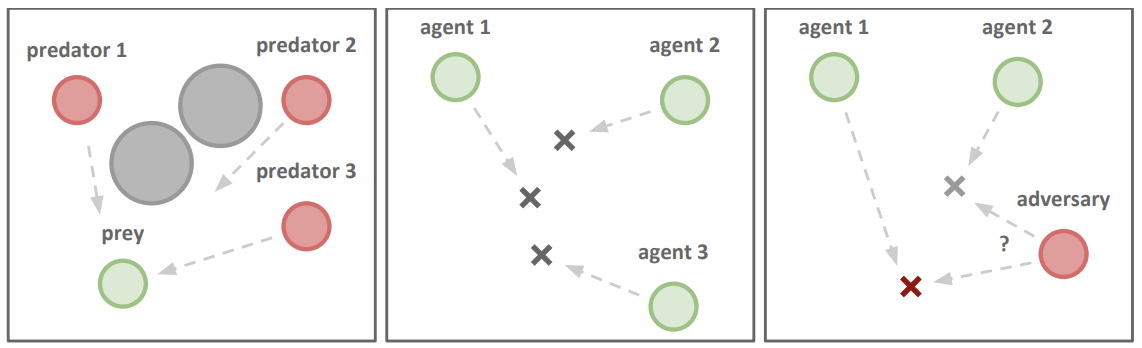} 
	\caption{Representations of tasks from the Particle World Environment \cite{lowe2017multi} that we study. (\emph{Left to Right}) Predator-Prey (PP), Cooperative Navigation (CN), and Physical Deception (PD). In PP, predators (red) seek to catch the prey (green) while avoiding obstacles (grey). 
	In CN, agents (green) each seek to navigate to a different landmark ($\times$) and are penalized for collisions with each other. In PD, one of the agents (green) must reach the true landmark (red $\times$), while preventing the adversary from reaching this landmark. In all tasks, rewards are \emph{sparse}. Agents receive a reward or penalty only when a corresponding reachability or collision criterion is satisfied. 
	}\label{PartWorldEnv}
\end{figure}

We examine three tasks from the \emph{Particle World} environment \cite{lowe2017multi} where multiple agents share a two-dimensional space with continuous states and actions. 
An illustration of the tasks is shown in Figure \ref{PartWorldEnv}, and we describe them below. 

\subsubsection{Predator-Prey} 
This task has $N$ predator agents who cooperate to capture $1$ faster-moving prey. 
Predators are rewarded when one of them collides with the prey, while the prey is penalized for the collision. 
The reward at other times is zero. 
Two landmarks impede movement of the agents. 

\subsubsection{Cooperative Navigation}
This task has $N$ agents and $N$ landmarks. 
Agents are each rewarded $r$ when an agent reaches a landmark, and penalized for collisions with each other. 
The reward at other times is zero. 
Therefore, the maximum rewards agents can obtain is $rN$.  
Thus, agents must learn to \emph{cover} the landmarks, and not collide with each other. 
\subsubsection{Physical Deception}
This task has $1$ adversary, $N$ agents, and $N$ landmarks. 
Only one landmark is the true target. 
Agents are rewarded when any one reaches the target, and penalized if the adversary reaches the target. 
At all other times, the agents get a reward of zero. 
An adversary also wants to reach the target, but it does not know which landmark is the target landmark. 
Thus, agents have to learn to split up and cover the landmarks to deceive the adversary. 

\begin{table*}[]
	\centering
	\begin{tabular}{|c|c|c|}
		\hline
		\textbf{Task}            & \textbf{$\phi_i(s_t,a^i_t,a^{-i}_t)$: SAM-Uniform} & \textbf{$\phi_i(s_t,a^i_t,a^{-i}_t)$: SAM-NonUniform} \\ \hline
		CN & $\alpha_{1}exp(-\beta_{1}\sum_{j=1}^N dist(s_t^j,L_j))$    & $ -M_{1}\theta_{{a^i_tL_i}}+\alpha_{2}exp(-\beta_{2}\sum_{j=1}^N dist(s_t^j,L_j))$ \\ \hline
		PD & $ \alpha_{3}exp(-\beta_{2}\sum_{j=1}^N dist(s_t^j,L_j))$    & $-M_{2}\theta_{{a^i_tL_i}}+\alpha_{4}exp(-\beta_{4}\sum_{j=1}^N dist(s_t^j,L_j))$\\ \hline
		PP & $ \alpha_{5}exp(-\beta_{5}\sum_{j=1}^N dist(s_t^{pred_j},s_t^{prey}))$   & $-M_{3}\sum_{j=1}^N \theta_{{a^{pred_j}_t s_t^{prey}}}+\alpha_{6}exp(-\beta_{6}\sum_{j=1}^N dist(s_t^{pred_j},s_t^{prey}))$ \\ \hline
	\end{tabular}\caption{Shaping advice, $F^i_t$ provided by SAM is given by Equation (\ref{LAPBA}) or (\ref{LBPBA}). The table lists the potential functions used in the Cooperative Navigation (CN), Physical Deception (PD), and Predator-Prey (PP) tasks. $L_j$ is the landmark to which agent $j$ is \emph{anchored} to. $dist(\cdot,\cdot)$ denotes the Euclidean distance. $\theta_{{a^j_tL_j}} \in [0, \pi]$ is the angle between the direction of the action taken by agent $j$ and the vector directed from its current position to $L_j$. In \emph{SAM-Uniform}, advice for every action of the agents for a particular $s_t$ is the same. 
		In \emph{SAM-NonUniform}, agents are additionally penalized if their actions are not in the direction of their target. In each case, $F^i_t$ is positive when agents take actions that move it towards their target.
		}\label{TableSAM}
\end{table*}

In each environment, SAM provides shaping advice to guide agents to obtain a higher positive reward. 
This advice is augmented to the reward received from the environment. 
The advice is a heuristic given by a difference of potential functions (Equations (\ref{LAPBA}) or (\ref{LBPBA})), and only needs to be specified once at the start of the training process. 

In the \emph{Cooperative Navigation} and \emph{Physical Deception} tasks, we \emph{anchor} each agent to a (distinct) landmark. 
The shaping advice will then depend on the distance of an agent to the landmark it is anchored to. 
Although distances computed in this manner will depend on the order in which the agents and landmarks are chosen, we observe that it empirically works across multiple training episodes where positions of landmarks and initial positions of agents are generated randomly. 
The advice provided by SAM is positive when agents move closer to landmarks they are anchored to. 
In the absence of anchoring, they may get distracted and move towards different landmarks at different time steps.  
are reset. 
Anchoring results in agents learning to cover landmarks faster.

We consider two variants of advice for each task. 
In \textbf{\emph{SAM-Uniform}}, the advice for every action taken is the same. 
In \textbf{\emph{SAM-NonUniform}}, a higher weight is given to some `good' actions over others for each $s_t$. 
We enumerate the advice for each task in Table \ref{TableSAM}. 
We use MADDPG as the base RL algorithm \cite{lowe2017multi}. 
We compare the performance of agents trained with SAM (SAM-Uniform or SAM-NonUniform) to the performance of agents trained using the sparse reward from the environment. 
We also compare the performance of SAM with a state-of-the-art reward redistribution technique called Iterative Relative Credit Assignment (IRCR), introduced in \cite{gangwani2020learning}. 
\begin{figure*}
	\begin{subfigure}{0.31\textwidth}
		\includegraphics[width=\linewidth]{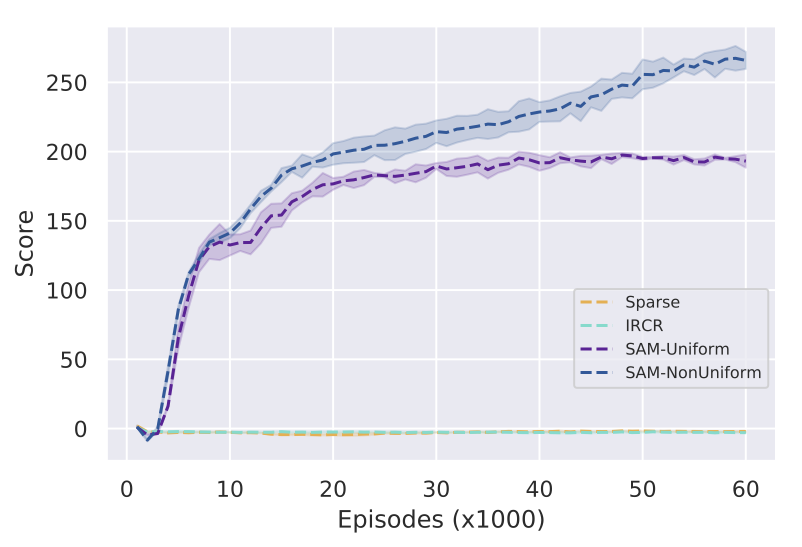}
		\caption{Cooperative Navigation ($N=6$)} \label{fig:1a}
	\end{subfigure}%
	\hspace*{\fill}   
	\begin{subfigure}{0.31\textwidth}
		\includegraphics[width=\linewidth]{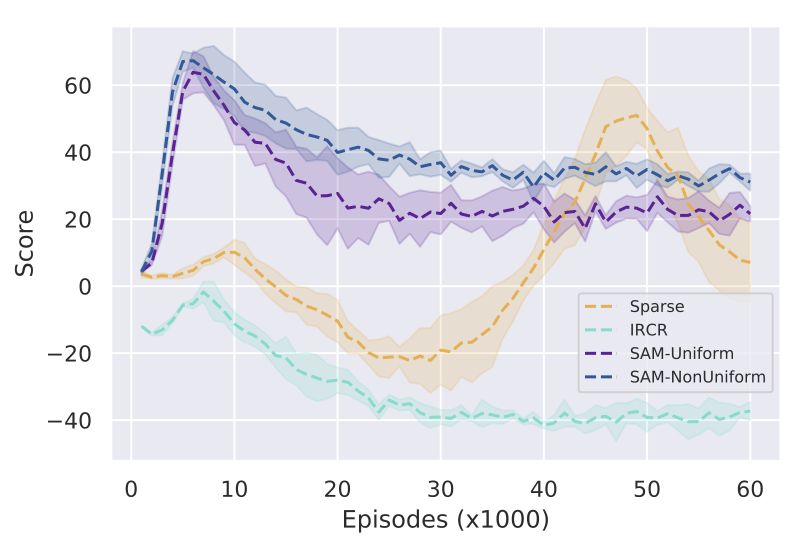}
		\caption{Physical Deception ($N=4$)} \label{fig:1b}
	\end{subfigure}%
	\hspace*{\fill}   
	\begin{subfigure}{0.31\textwidth}
		\includegraphics[width=\linewidth]{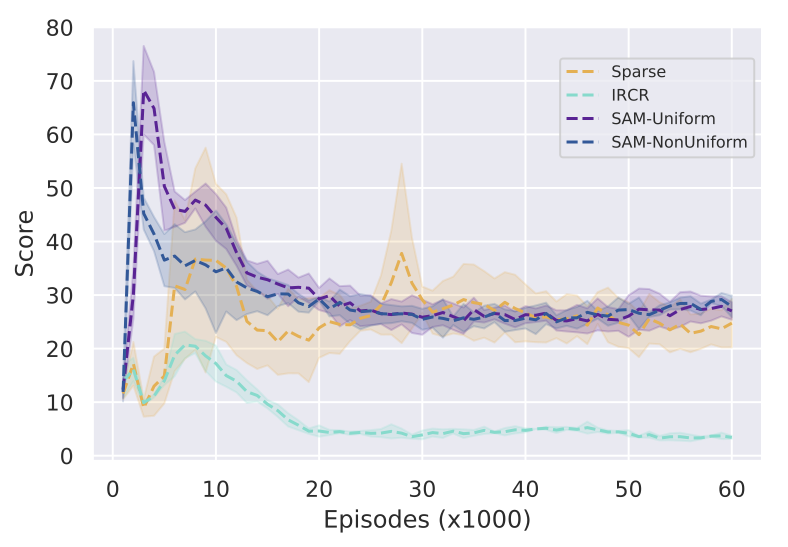}
		\caption{Predator-Prey ($3$ pred., $1$ prey.)} \label{fig:1c}
	\end{subfigure}
	
	\caption{Average and variance of scores when agents use SAM-NonUniform (blue), SAM-Uniform (purple), IRCR (green) and sparse rewards (orange). SAM-NonUniform results in the highest average scores. SAM-Uniform compares favorably, and both significantly outperform agents trained using only sparse rewards. 
	IRCR is not able to guide agents to obtain higher rewards in all three tasks.
	} \label{FigGraphs}
\end{figure*}
\subsection{Results}

Figure \ref{FigGraphs} shows the average and variance of the \textbf{score} during different stages of the training process.
The score for a task is the average agent reward in cooperative tasks, and the \emph{average agent advantage} ($=$ agent $-$ adversary reward) in competitive tasks \cite{wen2019probabilistic}. 

In terms of agent scores averaged over the last $1000$ training episodes,
agents equipped with \emph{SAM-NonUniform} have the best performance. 
This is because SAM-NonUniform provides specific feedback on the quality of agents' actions. 
\emph{SAM-Uniform} also performs well in these tasks. 

In cooperative navigation, when agents use only the sparse rewards from the environment, the agents are not able to learn policies that will allow them to even partially cover the landmarks. 
In comparison, SAM guides agents to learn to adapt to each others' policies, and cover all the landmarks. 
A similar phenomenon is observed in physical deception, where SAM guides agents to learn policies to cover the landmarks. 
This behavior of the agents is useful in deceiving the adversary from moving towards the true landmark, thereby resulting in lower final rewards for the adversary. 

We additionally compare the performance of SAM with a technique called IRCR that was introduced in \cite{gangwani2020learning}. 
We observe that agents using IRCR receive the lowest scores in all three tasks. 
We believe that a possible reason for this is that in each training episode, IRCR accumulates rewards till the end of the episode, and then uniformly redistributes the accumulated reward along the length of the episode. 
A consequence of this is that an agent may find it more difficult to identify the time-step when it reaches a landmark or when a collision occurs. 
For example, in the \emph{Predator-Prey} task, suppose that the length of an episode is $T_{ep}$. 
Consider a scenario where one of the predators collides with the prey at a time $T < T_{ep}$, and subsequently moves away from the prey. 
When IRCR is applied to this scenario, the redistributed reward at time $T$ will be the same as that at other time steps before $T_{ep}$. 
This property makes it difficult to identify critical time-steps when collisions between agents happen.  

The authors of \cite{lowe2017multi} observed that agent policies being unable to adapt to each other in competitive environments resulted in oscillations in rewards. 
Figure \ref{fig:1b} indicates that SAM is able to alleviate this problem. 
Policies learned by agents using SAM in the physical deception task 
result in much smaller oscillations in the rewards than when using sparse rewards alone. 

\section{Conclusion}\label{Conclusion}

This paper presented SAM, a framework to incorporate domain knowledge through shaping advice in cooperative and competitive multi-agent reinforcement learning (MARL) environments with sparse rewards. 
The shaping advice for each agent was a heuristic specified as a difference of potential functions, and was augmented to the reward provided by the environment. 
The modified reward signal provided agents with immediate feedback on the quality of the actions taken at each time-step. 
SAM used the centralized training with decentralized paradigm to efficiently learn decentralized policies for each agent that used only their individual local observations. 
We showed through theoretical analyses and experimental validation that shaping advice provided by SAM did not distract agents from accomplishing task objectives specified by the environment reward. 
We observed that SAM accelerated the learning of policies, and resulted in improved agent performance in three tasks with sparse rewards in the multi-agent Particle World environment. 
In competitive tasks, SAM alleviated a known problem of `oscillating rewards’ seen in prior work. 

Future work will extend SAM to cases where the shaping advice can be adaptively learned instead of being fixed for the duration of training. 
We will also analyze the sample-efficiency of learning when agents are equipped with SAM. 
This will broaden the application of SAM to more challenging real-world MARL environments. 

\bibliographystyle{IEEEtran}
\bibliography{RewShapBib}
\end{document}